\documentclass[letterpaper]{article} 
\usepackage[preprint]{aaai2027}
\usepackage[hyphens]{url}  
\usepackage{graphicx} 
\usepackage{natbib}  
\usepackage{caption} 
\usepackage{algorithm}
\usepackage{algorithmic}
\usepackage{booktabs}
\usepackage{amsmath,amssymb,amsfonts}
\usepackage{amsthm}
\newtheorem{proposition}{Proposition}
\newtheorem{lemma}{Lemma}
\usepackage{multirow}
\usepackage[table]{xcolor}
\usepackage{array}
\usepackage{tikz}
\usetikzlibrary{arrows.meta}
\definecolor{tpInk}{HTML}{293B49}
\definecolor{tpMuted}{HTML}{61717C}
\definecolor{tpLine}{HTML}{BCC9D1}
\definecolor{tpBlue}{HTML}{477FA6}
\definecolor{tpBlueLight}{HTML}{EAF2F8}
\definecolor{tpRed}{HTML}{B75249}
\definecolor{tpRedLight}{HTML}{FBEEEA}
\definecolor{tpGreen}{HTML}{388575}
\definecolor{tpGreenLight}{HTML}{E8F3ED}
\definecolor{tpPurple}{HTML}{8071A3}
\definecolor{tpPurpleLight}{HTML}{F0EDF7}
\definecolor{tpGold}{HTML}{B8883B}
\tikzset{
 tptext/.style={font=\sffamily\fontsize{8}{9.2}\selectfont,text=tpInk,align=center},
 tparrow/.style={-{Stealth[length=3.1pt,width=2.7pt]},draw=tpInk,line width=.65pt,rounded corners=2pt},
 tpfeedback/.style={tparrow,draw=tpPurple,line width=.8pt},
 pics/tpmodel/.style args={#1}{code={
  \draw[draw=#1,fill=#1!12,line width=.7pt,rounded corners=2pt] (-10,-6) rectangle (10,8);
  \draw[draw=#1,line width=.7pt] (-10,0)--(-13,0)--(-13,4)--(-10,4) (10,0)--(13,0)--(13,4)--(10,4);
  \draw[draw=#1,line width=.7pt] (0,8)--(0,11);
  \fill[#1] (0,12) circle (1.3pt);
  \draw[draw=#1,fill=white,rounded corners=1pt,line width=.55pt] (-6,-1) rectangle (6,5);
  \fill[#1] (-3,2) circle (.9pt) (3,2) circle (.9pt);
  \draw[draw=#1,line width=.6pt] (-4,-6)--(-4,-9)--(4,-9)--(4,-6);
 }},
 pics/tppool/.style={code={
  \draw[draw=tpLine,fill=tpBlueLight,line width=.6pt,rounded corners=1pt] (-7,-7) rectangle (10,12);
  \draw[draw=tpLine,fill=white,line width=.6pt,rounded corners=1pt] (-10,-10) rectangle (7,9);
  \draw[draw=tpBlue,fill=white,line width=.7pt,rounded corners=1pt] (-13,-13) rectangle (4,6);
  \node[tptext,font=\sffamily\bfseries\fontsize{9}{9}\selectfont,text=tpBlue] at (-4.5,0) {$x$};
  \draw[tpLine,line width=.6pt] (-10,-5)--(1,-5) (-10,-8)--(-2,-8);
 }},
 pics/tpverify/.style={code={
  \draw[draw=tpGreen,fill=tpGreenLight,line width=.7pt,rounded corners=1.5pt] (-8,-10) rectangle (8,10);
  \draw[draw=tpGreen,fill=white,line width=.6pt,rounded corners=.6pt] (-3,8) rectangle (3,12);
  \draw[draw=tpGreen,line width=1pt,line cap=round,line join=round] (-5,1)--(-1,-3)--(5,4);
  \draw[draw=tpGreen!45,line width=.6pt] (-4,-7)--(4,-7);
 }},
 pics/tprank/.style={code={
  \foreach \y/\tpcol/\len in {8/tpRed/12,0/tpRed/9,-8/tpLine/4}{
   \draw[draw=\tpcol,fill=\tpcol!10,line width=.6pt,rounded corners=1pt] (-12,\y-3) rectangle (12,\y+3);
   \fill[\tpcol] (-9,\y-1.1) rectangle (-9+\len,\y+1.1);
  }
 }},
 pics/tpdensity/.style={code={
  \draw[draw=tpPurple!45,fill=white,line width=.65pt,rounded corners=1.5pt] (-15,-12) rectangle (15,13);
  \draw[draw=tpLine,line width=.5pt] (-11,-8)--(11,-8) (-11,-8)--(-11,8);
  \fill[tpPurple!16] (-11,-8) .. controls (-8,-8) and (-7,5) .. (-3,6) .. controls (1,6) and (1,-8) .. (11,-8)--cycle;
  \draw[draw=tpPurple,line width=.9pt] (-11,-8) .. controls (-8,-8) and (-7,5) .. (-3,6) .. controls (1,6) and (1,-8) .. (11,-8);
  \draw[draw=tpBlue,line width=.7pt] (-11,-8) .. controls (-5,-8) and (-4,1) .. (3,2) .. controls (8,2) and (8,-8) .. (11,-8);
 }},
 pics/tpupdate/.style={code={
  \foreach \ya in {-8,0,8}{\foreach \yb in {-5,5}{\draw[tpRed!45,line width=.55pt] (-10,\ya)--(0,\yb);}}
  \foreach \ya in {-5,5}{\foreach \yb in {-8,0,8}{\draw[tpRed!45,line width=.55pt] (0,\ya)--(10,\yb);}}
  \foreach \y in {-8,0,8}{\draw[draw=tpRed,fill=tpRedLight,line width=.65pt] (-10,\y) circle (2pt);\draw[draw=tpRed,fill=tpRedLight,line width=.65pt] (10,\y) circle (2pt);}
  \foreach \y in {-5,5}{\draw[draw=tpRed,fill=tpRed!35,line width=.65pt] (0,\y) circle (2pt);}
 }},
 pics/tpgroup/.style args={#1}{code={
  \draw[draw=tpLine,fill=white,line width=.55pt,rounded corners=1pt] (-8,-11) rectangle (9,11);
  \draw[tpLine,line width=.55pt] (-5,8)--(6,8);
  \foreach \j in {0,...,7}{
   \pgfmathtruncatemacro{\col}{mod(\j,2)}
   \pgfmathtruncatemacro{\row}{floor(\j/2)}
   \begin{scope}[shift={(-4+\col*7,4-\row*3.7)}]
    \ifnum\j<#1\relax
      \draw[draw=tpGreen,line width=.75pt,line cap=round,line join=round] (-1.3,0)--(-.3,-1)--(1.5,1.3);
    \else
      \draw[draw=tpRed!85,line width=.65pt,line cap=round] (-1,-1)--(1,1) (-1,1)--(1,-1);
    \fi
   \end{scope}
  }
 }}
}

\definecolor{cOurs}{RGB}{249,214,210}
\definecolor{cOursLite}{RGB}{250,240,239}
\definecolor{cGroup}{RGB}{200,205,214}

\graphicspath{{figures/}}

\newcommand{\method}{ThinkPrior}

\newcommand{\numLsil}{0.217}
\newcommand{\numLacc}{0.582}
\newcommand{\numLwst}{552}

\newcommand{\ua}{$\uparrow$}
\newcommand{\da}{$\downarrow$}
\newcommand{\up}[1]{\rlap{\textsubscript{\fontsize{6}{6}\selectfont\textcolor{green!55!black}{#1}}}}
\newcommand{\dn}[1]{\rlap{\textsubscript{\fontsize{6}{6}\selectfont\textcolor{red!75!black}{#1}}}}
\newcommand{\best}[1]{\textbf{#1}}

\newif\ifprovhl
\provhlfalse

\newcommand{\lpOurSilentA}{0.121}                  
\newcommand{\lpOurSilentB}{0.397}                  
\newcommand{\lpOurWaste}{5085}                     
\newcommand{\lpOurAcc}{0.596}                      
\newcommand{\lpOurSilentASeeds}{0.150/0.113/0.100}
\newcommand{\lpOurSilentCSeeds}{0.092/0.075/0.067} 

\newcommand{\lpNpSilentA}{0.242}                   
\newcommand{\lpNpSilentASeeds}{0.287/0.175/0.263}
\newcommand{\lpNpSilentCSeeds}{0.250/0.225/0.212}  

\newcommand{\lpNpSilentB}{0.476}                   
\newcommand{\lpNpSilentBSeeds}{0.507/0.477/0.445}
\newcommand{\lpOurSilentBSeeds}{0.455/0.388/0.349}
\newcommand{\lpNpWaste}{6099}                      
\newcommand{\lpNpWasteSeeds}{6488/6112/5696}
\newcommand{\lpOurWasteSeeds}{5824/4960/4472}
\newcommand{\lpNetPct}{16.6\%}                     
\newcommand{\lpNpAcc}{0.411}                       
\newcommand{\lpNpAccSensitivity}{0.604}            
\newcommand{\lpNpAccSeeds}{0.588/0.620/0.024}

\title{\method{}: Zero-Rollout Difficulty Priors\\
for Cold-Start Prompt Selection in RLVR}

\author{
    Tommy Sha\textsuperscript{\rm 1},
    Skylar Zhai\textsuperscript{\rm 2},
    Siqi Zhao\textsuperscript{\rm 2}
}
\affiliations{
    \textsuperscript{\rm 1}Stony Brook University
    \qquad
    \textsuperscript{\rm 2}University of Minnesota Twin Cities\\[-0.2em]
    {\scriptsize\ttfamily
    Correspondence: tianming.sha@stonybrook.edu;
    \{haoti002,zhao2052\}@umn.edu}
}

\begin{document}
\maketitle

\begin{abstract}
GRPO spends generation on response groups whose rewards are all identical, even though these
groups provide no relative reward signal. Prompt selectors can avoid some of this waste by
learning from policy outcomes, but must first collect the history used to make those decisions.
We introduce \method{}, which initializes prompt selection with verified responses from a
smaller external model. A reusable offline pass supplies per-prompt Beta priors; their
closed-form expected learnability scores favor prompts likely to produce mixed rewards.
Target-policy outcomes then update the estimates during training. On Qwen2.5-Math-7B,
\method{} more than halves early silent groups while retaining comparable final accuracy.
Combining it with DAPO reduces policy rollout generation by $10.6\%$ at the same update
budget and mean MATH500 accuracy. These results show how external difficulty information
can reduce the cost of selecting useful training prompts before policy history is available.
\end{abstract}

\section{Introduction}

Reinforcement learning with verifiable rewards (RLVR) has improved the reasoning capabilities
of large language models by learning from automatically checked solutions
\cite{guo2025deepseekr1}. GRPO \cite{shao2024deepseekmath} is a widely used approach in this
setting: it generates several responses to each prompt and updates the policy using their
relative rewards. Repeating this process across many prompts requires substantial generation,
making rollout computation a major cost of training.

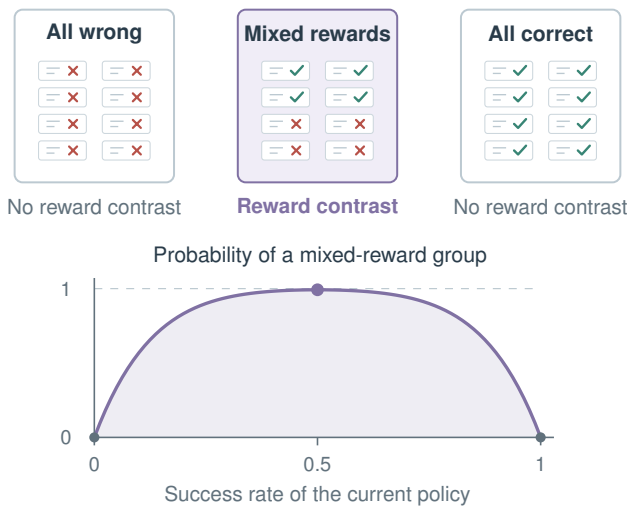
\begin{figure}[t]
\centering
\begingroup
\begin{tikzpicture}[x=1pt,y=1pt,every node/.style={tptext}]
\path[use as bounding box] (0,0) rectangle (237,222);
\node[font=\sffamily\bfseries\fontsize{9}{10}\selectfont] at (118,211) {Reward differences provide the learning signal};
\foreach \cx/\title/\successes/\border/\bg in {34/All wrong/0/tpLine/white,118/Mixed rewards/4/tpPurple/tpPurpleLight,202/All correct/8/tpLine/white}{
 \draw[draw=\border,fill=\bg,line width=.75pt,rounded corners=3pt] (\cx-30,124) rectangle (\cx+30,189);
 \node[font=\sffamily\bfseries\fontsize{8}{9}\selectfont] at (\cx,180) {\title};
 \foreach \j in {0,...,7}{
  \pgfmathtruncatemacro{\c}{mod(\j,2)}
  \pgfmathtruncatemacro{\r}{floor(\j/2)}
  \begin{scope}[shift={(\cx-12+24*\c,166-10*\r)}]
   \draw[draw=tpLine!75,fill=white,line width=.4pt,rounded corners=1pt] (-9,-3.6) rectangle (9,3.6);
   \draw[draw=tpLine,line width=.5pt] (-6,1)--(-1,1) (-6,-1)--(-3,-1);
   \ifnum\j<\successes\relax
    \draw[draw=tpGreen,line width=.9pt,line cap=round,line join=round] (2,0)--(3.5,-1.5)--(6.5,1.8);
   \else
    \draw[draw=tpRed,line width=.85pt,line cap=round] (2.5,-1.5)--(5.5,1.5) (2.5,1.5)--(5.5,-1.5);
   \fi
  \end{scope}
 }
}
\node[text=tpMuted,font=\sffamily\fontsize{7.8}{8.5}\selectfont] at (34,115) {No reward contrast};
\node[text=tpPurple,font=\sffamily\bfseries\fontsize{7.8}{8.5}\selectfont] at (118,115) {Reward contrast};
\node[text=tpMuted,font=\sffamily\fontsize{7.8}{8.5}\selectfont] at (202,115) {No reward contrast};
\node[text=tpInk,font=\sffamily\fontsize{8}{9}\selectfont] at (119,96) {Probability of a mixed-reward group};
\draw[draw=tpLine,line width=.45pt,dashed] (34,84)--(202,84);
\path[fill=tpPurple!13] (34,28) -- plot[domain=0:1,samples=81,variable=\x]
 ({34+168*\x},{28+56*(1-pow(\x,8)-pow(1-\x,8))}) -- (202,28) -- cycle;
\draw[draw=tpMuted,line width=.6pt] (34,88)--(34,28)--(207,28);
\draw[draw=tpPurple,line width=1.25pt] plot[domain=0:1,samples=81,variable=\x]
 ({34+168*\x},{28+56*(1-pow(\x,8)-pow(1-\x,8))});
\fill[tpMuted] (34,28) circle (1.9pt) (202,28) circle (1.9pt);
\fill[tpPurple] (118,83.5625) circle (2.4pt);
\node[anchor=east,text=tpMuted,font=\sffamily\fontsize{7.5}{8.5}\selectfont] at (29,28) {0};
\node[anchor=east,text=tpMuted,font=\sffamily\fontsize{7.5}{8.5}\selectfont] at (29,84) {1};
\foreach \xx/\label in {34/0,118/0.5,202/1}{
 \draw[draw=tpMuted,line width=.6pt] (\xx,28)--(\xx,25);
 \node[text=tpMuted,font=\sffamily\fontsize{7.5}{8.5}\selectfont] at (\xx,18) {\label};
}
\node[text=tpMuted,font=\sffamily\fontsize{7.8}{8.5}\selectfont] at (118,6) {Success rate of the current policy};
\end{tikzpicture}
\endgroup
\caption{Which prompts provide a GRPO learning signal? All-wrong and all-correct groups have
uniform rewards and zero relative advantage. Mixed rewards provide a within-group comparison.
The curve gives the probability of mixed rewards for the illustrated group size; the cards
show example outcomes. \method{} estimates this probability before selecting prompts.}
\label{fig:teaser}
\end{figure}

To reduce rollout cost, DAPO \cite{yu2025dapo} retains response groups with varied rewards
and refills the update batch as needed. GRESO \cite{zheng2025greso} uses past rewards to
probabilistically skip prompts before generation, while MoPPS \cite{qu2025mopps} selects
prompts using success-rate posteriors. Predictive selectors can also share observations
across prompts \cite{qu2026gps}.

These methods obtain their selection signals from target-policy rollouts. Filtering pays
for responses that it later discards, while history-based selection must collect observations
before they can guide subsequent choices. At cold start, this acquisition cost is part of
the training budget. When every response to a prompt is correct or every response is
incorrect, the group provides no reward comparison (Figure~\ref{fig:teaser}); we call these
\emph{silent groups}. Early rollouts can therefore inform selection without supplying a
group-relative learning signal. We ask whether an external model can provide useful initial
difficulty information before target-policy generation, especially for short fine-tuning runs.

The quantity to estimate follows from GRPO's normalization. Every non-silent group has the
same squared-advantage sum under exact normalization, so the expected sum is proportional to
the probability of obtaining mixed rewards. We call this probability \emph{learnability}.
An intermediate estimate of the success rate alone is insufficient: a mean near one half
can describe either a prompt known to yield varied outcomes or a poorly known prompt that
may be consistently solved or failed. The selector should therefore average learnability
over its uncertainty about the success rate.

\method{} uses a smaller, off-the-shelf \emph{anchor} to answer the prompt pool and converts
its verified success rates into finite-strength Beta priors. The resulting posterior
expectation has a closed form, enabling prompt ranking from the first training step.
The target policy's subsequent responses provide both the GRPO update and new evidence for
selection (Figure~\ref{fig:framework}). This separates the acquisition of initial difficulty
information from target-policy generation: the anchor pass can be reused across runs, while
online feedback adapts the estimates to each run.

We evaluate whether this initialization improves early rollout utilization and whether its
combination with dynamic sampling reduces generation for a fixed update budget. Our
contributions are threefold:
\begin{itemize}\itemsep2pt
\item We derive an expected-learnability score and establish how reward normalization and
      uncertainty in success rates determine its prompt-selection behavior.
\item We introduce \method{}, which initializes the score with verified external-model
      responses and updates it using training feedback.
\item We demonstrate improved rollout utilization with comparable task performance.
      \method{} more than halves early silent groups and, with DAPO, uses $10.6\%$ fewer
      policy rollouts at the same update budget and mean MATH500 accuracy.
\end{itemize}

\section{Method}

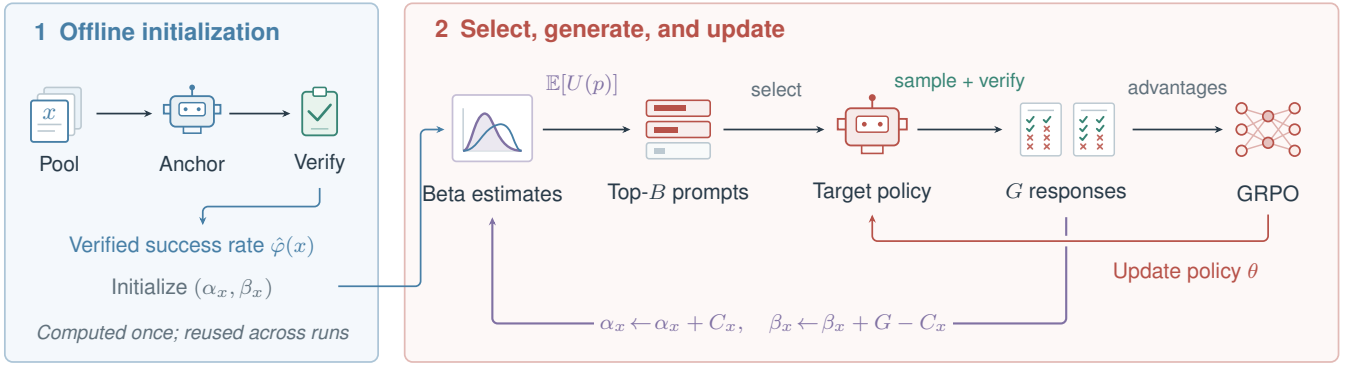
\begin{figure*}[t]
\centering
\begin{tikzpicture}[x=1pt,y=1pt,every node/.style={tptext}]
\path[use as bounding box] (0,0) rectangle (503,136);
\draw[draw=tpBlue!45,fill=tpBlueLight!50,line width=.65pt,rounded corners=3pt] (.5,.5) rectangle (141,135.5);
\draw[draw=tpRed!40,fill=tpRedLight!40,line width=.65pt,rounded corners=3pt] (151,.5) rectangle (502.5,135.5);
\node[anchor=west,text=tpBlue,font=\sffamily\bfseries\fontsize{9}{10}\selectfont] at (8,125) {1\enspace Offline initialization};
\node[anchor=west,text=tpRed,font=\sffamily\bfseries\fontsize{9}{10}\selectfont] at (159,125) {2\enspace Select, generate, and update};
\pic[scale=.82] at (21,94) {tppool};
\pic[scale=.82] at (71,94) {tpmodel=tpBlue};
\pic[scale=.82] at (119,94) {tpverify};
\draw[tparrow] (35,94)--(56,94);
\draw[tparrow] (85,94)--(108,94);
\node[font=\sffamily\fontsize{7.8}{9}\selectfont] at (21,75) {Pool};
\node[font=\sffamily\fontsize{7.8}{9}\selectfont] at (71,75) {Anchor};
\node[font=\sffamily\fontsize{7.8}{9}\selectfont] at (119,75) {Verify};
\draw[tparrow,draw=tpBlue] (119,66)--(119,59)--(75,59)--(75,52);
\node[text=tpBlue] at (71,44) {Verified success rate $\hat\varphi(x)$};
\node[text=tpMuted,font=\sffamily\fontsize{8}{9}\selectfont] at (71,28) {Initialize $(\alpha_x,\beta_x)$};
\node[text=tpMuted,font=\sffamily\itshape\fontsize{7.3}{8}\selectfont] at (71,11) {Computed once; reused across runs};
\draw[tparrow,draw=tpBlue] (126,29)--(157,29)--(157,88)--(166,88);
\pic at (184,88) {tpdensity};
\pic at (254,88) {tprank};
\pic at (327,88) {tpmodel=tpRed};
\pic[scale=.9] at (390,88) {tpgroup=3};
\pic[scale=.9] at (410,88) {tpgroup=5};
\pic at (476,88) {tpupdate};
\draw[tparrow] (203,88)--(236,88);
\draw[tparrow] (271,88)--(309,88);
\draw[tparrow] (344,88)--(376,88);
\draw[tparrow] (425,88)--(459,88);
\node[text=tpPurple] at (218,105) {$\mathbb E[U(p)]$};
\node[text=tpMuted,font=\sffamily\fontsize{7.3}{8}\selectfont] at (291,103) {select};
\node[text=tpGreen,font=\sffamily\fontsize{7.3}{8}\selectfont] at (360,106) {sample + verify};
\node[text=tpMuted,font=\sffamily\fontsize{7.3}{8}\selectfont] at (442,103) {advantages};
\node at (184,64) {Beta estimates};
\node at (254,64) {Top-$B$ prompts};
\node at (327,64) {Target policy};
\node at (400,64) {$G$ responses};
\node at (476,64) {GRPO};
\draw[tparrow,draw=tpRed] (476,56)--(476,46)--(327,46)--(327,56);
\node[text=tpRed,font=\sffamily\fontsize{8}{9}\selectfont] at (445,34) {Update policy $\theta$};
\draw[draw=tpPurple,line width=.8pt] (400,55)--(400,48);
\draw[tpfeedback] (400,44)--(400,15)--(184,15)--(184,55);
\node[text=tpPurple,fill=tpRedLight!40,inner sep=2pt] at (290,15) {$\alpha_x\!\gets\!\alpha_x+C_x,\quad\beta_x\!\gets\!\beta_x+G-C_x$};
\end{tikzpicture}
\caption{\method{} workflow. An offline anchor pass initializes per-prompt Beta estimates.
Online selection ranks their expected learnability; verified responses from the chosen prompts
update both the policy and the selected prompts' success/failure counts.}
\label{fig:framework}
\end{figure*}

\subsection{Prompt utility under GRPO}

Let $\mathcal X$ be a prompt pool, $\pi_\theta$ the current policy, and $v$ a binary verifier.
For a prompt $x$, write $p(x)=\Pr_{y\sim\pi_\theta(\cdot\mid x)}[v(y)=1]$ for its success rate.
GRPO generates $G$ independent responses $y_1,\ldots,y_G$ and computes group-relative
advantages from their rewards $r_i=v(y_i)$:
\begin{equation}
A_i=\frac{r_i-\bar r}{\operatorname{std}(r)+\varepsilon},
\qquad \bar r=\frac1G\sum_{j=1}^{G}r_j,
\label{eq:adv}
\end{equation}
using the sample standard deviation and $\varepsilon=10^{-4}$. If all responses receive the
same reward, every $A_i$ is zero. Such silent groups supply no reward-advantage signal;
our main experiments use this term without a reference-KL penalty.

At success rate $p$, the silent-group probability and its complement are
\begin{equation}
s(p)=p^G+(1-p)^G,\qquad U(p)=1-s(p).
\label{eq:learn}
\end{equation}
We call $U(p)$ the prompt's \emph{learnability}, the probability of obtaining both correct
and incorrect responses. The following result connects this probability to the learning
signal supplied by a prompt, motivating its use as the selection objective.

\begin{proposition}[Silent groups and advantage mass]
\label{prop:silent}
Let $G\ge2$ and $C\sim\mathrm{Binomial}(G,p)$. Then \textup{(i)} the group is silent exactly when
$C\in\{0,G\}$, with probability $s(p)$; \textup{(ii)} $U$ is strictly concave on $[0,1]$,
symmetric about $p=1/2$, strictly decreasing in $|p-1/2|$, and has maximum
$U(1/2)=1-2^{1-G}$, with zeros exactly at $p\in\{0,1\}$; \textup{(iii)} for a non-silent group,
$\sum_i A_i^2=G-1$ when $\varepsilon=0$, so
$\mathbb E[\sum_i A_i^2]=(G-1)U(p)$.
\end{proposition}

The proof is in Appendix~A. Under exact normalization, a group with one correct response
and a group with evenly split rewards both contribute $G-1$ to the squared-advantage sum.
Difficulty changes how often a prompt supplies such a group. For $G=8$, this probability is
$U(1/2)=0.992$, compared with $U(0.05)=U(0.95)=0.337$, motivating selection around the
policy's intermediate success-rate region.

\subsection{Selection with uncertain success rates}

The policy's true success rate is unknown. We represent it by a per-prompt estimate
$\mathrm{Beta}(\alpha_x,\beta_x)$, whose mean estimates the success rate and whose total
pseudo-count mass controls concentration. Averaging the utility in Eq.~\eqref{eq:learn}
under this estimate gives
\begin{equation}
\begin{split}
U_\beta(x)
&=\mathbb E_{p\sim\mathrm{Beta}(\alpha_x,\beta_x)}[U(p)]\\
&=1-\frac{(\alpha_x)_G}{(\alpha_x+\beta_x)_G}
    -\frac{(\beta_x)_G}{(\alpha_x+\beta_x)_G},
\end{split}
\label{eq:ubeta}
\end{equation}
where $(a)_G=\prod_{j=0}^{G-1}(a+j)$ is the rising factorial. The Beta moment identity
(Appendix~A, Lemma~\ref{lem:beta}) gives this closed form, which requires only the two stored counts
and the group size.

Evaluating $U$ at the estimated mean would discard uncertainty about that mean. The next
result explains why the posterior expectation makes a different selection decision.

\begin{proposition}[Posterior concentration]
\label{prop:disp}
Let $\alpha_x,\beta_x>0$, $\mu=\alpha_x/(\alpha_x+\beta_x)\in(0,1)$, and
$m=\alpha_x+\beta_x$. For every finite $m$, $U_\beta(x)<U(\mu)$. At fixed $\mu$,
$U_\beta(x)$ is strictly increasing in $m$, tending to $0$ as $m\to0$ and to $U(\mu)$
as $m\to\infty$.
\end{proposition}

The proof in Appendix~A uses the strict concavity of $U$ and the Beta moments. At equal
mean success rates, the posterior score favors the prompt whose difficulty is better known.
A diffuse estimate can have an intermediate mean while still assigning substantial
probability to success rates near zero or one, where silent groups are likely. This separates
useful variation in a prompt's responses from uncertainty about its success rate, and is
why we use $\mathbb E[U(p)]$ to rank prompts rather than $U(\mathbb E[p])$.

Applying Proposition~\ref{prop:silent} under exact normalization gives
\begin{equation}
\mathbb E_{p,r}\!\left[\sum_i A_i^2\right]=(G-1)U_\beta(x).
\label{eq:posterior-mass}
\end{equation}
For a fixed group size and prompt-batch budget, selecting the $B$ highest-scoring prompts
maximizes this expected squared-advantage sum for the next batch. We break score ties by
pool order. Policy optimization combines these advantages with the generated responses'
score-function gradients; the experiments evaluate the resulting task performance.

\subsection{Initializing and updating the estimates}

The score requires initial information about each prompt. The following result shows why
an uninformative prior cannot provide this ranking.

\begin{samepage}
\begin{proposition}[Cold-start initialization]
\label{prop:cold}
Under $\alpha_x=\beta_x=1$, the posterior expected learnability is
$U_\beta(x)=1-2/(G+1)$ for every prompt. More generally, a prompt-independent initialization
makes $U_\beta$ constant over the pool.
\end{proposition}
\end{samepage}

The proof is in Appendix~A. Changing shared pseudo-counts changes every prompt's score in
the same way; it cannot identify which prompts will yield reward variation. Initial estimates
must therefore depend on the prompts themselves. We obtain this information from an external
anchor, whose verified responses are available before any target-policy rollout.

\paragraph{Offline initialization.}
A smaller instruction-tuned anchor $T$ generates $k$ responses per prompt. The training
verifier gives its empirical success rate:
\begin{equation}
\hat\varphi(x)=\frac1k\sum_{j=1}^{k}v(\tilde y_j),
\qquad \tilde y_j\sim T(\cdot\mid x).
\label{eq:prior}
\end{equation}
We convert this estimate into initial Beta parameters:
\begin{equation}
\alpha_x=\kappa\hat\varphi(x)+\epsilon_0,\qquad
\beta_x=\kappa\bigl(1-\hat\varphi(x)\bigr)+\epsilon_0.
\label{eq:init}
\end{equation}
Here $\kappa$ sets the weight of the external observations, and $\epsilon_0>0$ permits
initialization when all anchor responses are correct or all are incorrect. We use $k=16$,
$\kappa=4$, and $\epsilon_0=10^{-3}$. The anchor measurements thus contribute approximately
four pseudo-observations, discounting them because they describe a different policy. The
prior strength also determines the initial concentration $m_0=\kappa+2\epsilon_0$, linking
its effect on early selection to Proposition~\ref{prop:disp} as well as to its weight in
subsequent updates.

Proposition~\ref{prop:silent} clarifies what must transfer from the anchor: its rates should
place useful prompts near the target policy's learnable region. Preserving a global
easy-to-hard ordering can be insufficient if the success-rate scale shifts across models.
We examine this distinction through the alternative-prior comparisons.

\paragraph{Online updates.}
For each selected prompt, the policy generates $G$ responses and the verifier counts the
correct responses $C_x$. We update
\begin{equation}
\alpha_x\leftarrow\alpha_x+C_x,\qquad
\beta_x\leftarrow\beta_x+G-C_x.
\label{eq:update}
\end{equation}
Only selected prompts receive new observations; the other counts are carried forward.
The trainer then performs a GRPO update on the generated responses. Prompts remain in the
pool and can be selected again. Algorithm~\ref{alg:tp} gives the complete procedure.

Each observed group changes both the estimated mean and concentration. Using the notation
of Proposition~\ref{prop:disp}, Eq.~\eqref{eq:update} gives
\begin{equation}
\mu'=\frac{m\mu+C_x}{m+G},\qquad m'=m+G.
\label{eq:mean-update}
\end{equation}
The fixed-mean concentration effect is thus only part of the online evolution. Repeatedly
solving a prompt moves its mean toward one and eventually lowers its probability of a
mixed-reward group, even as evidence accumulates. The selector can then shift toward other
prompts. The estimate pools observations from successive policy checkpoints, so it reflects
the prompt's training history as well as the latest group's outcome.

\begin{algorithm}[t]
\caption{\method{} prompt selection}
\label{alg:tp}
\begin{algorithmic}[1]
\REQUIRE pool $\mathcal X$, policy $\pi_\theta$, anchor $T$, verifier $v$, group size $G$, batch size $B$, horizon $T_{\max}$
\STATE \textbf{Offline initialization:} $k=16$, $\kappa=4$, $\epsilon_0=10^{-3}$
\FORALL{$x\in\mathcal X$}
  \STATE Generate and verify $k$ anchor responses
  \STATE Compute $\hat\varphi(x)$ using Eq.~\eqref{eq:prior}
  \STATE Initialize $(\alpha_x,\beta_x)$ using Eq.~\eqref{eq:init}
\ENDFOR
\STATE \textbf{Training:}
\FOR{$t=1$ \TO $T_{\max}$}
  \STATE Score every prompt using Eq.~\eqref{eq:ubeta}
  \STATE $\mathcal B_t\leftarrow$ top-$B$ prompts; ties by pool order
  \FORALL{$x\in\mathcal B_t$}
    \STATE Generate $G$ responses from $\pi_\theta(\cdot\mid x)$
    \STATE Verify responses and count successes $C_x$
    \STATE Update $(\alpha_x,\beta_x)$ using Eq.~\eqref{eq:update}
  \ENDFOR
  \STATE Update $\pi_\theta$ with GRPO on the sampled responses
\ENDFOR
\end{algorithmic}
\end{algorithm}

\subsection{Rollout budgets and dynamic sampling}

The basic procedure generates $BG$ responses per step. With this fixed budget, better selection
increases the proportion of responses belonging to non-silent groups. To reduce the number of
responses generated for a fixed useful update batch, we combine the selector with DAPO's
dynamic sampling \cite{yu2025dapo}. In this composition, candidate prompts come from the
highest-scoring $4B$ prompts; the trainer retains its existing filtering and refill procedure
and uses $B$ non-silent groups for each update. Improving the candidate pool reduces the need
to generate and discard silent groups. Table~\ref{tab:money} reports the measured generation
cost of this composition.

For $N$ prompts, the offline pass requires $Nk$ anchor generations, reusable across runs.
Online scoring takes $O(NG)$ scalar operations and stores two counts per prompt. Anchor
preparation and policy generation are accounted for separately.

\section{Experimental Setup}

\paragraph{Models and data.}
We train Qwen2.5-Math-7B (base) \cite{yang2024qwen25math} on $250$ problems from the MATH
training split \cite{hendrycks2021math}, covering its five difficulty levels and seven
subjects. Training uses GRPO with LoRA for $60$ steps, selecting $B=8$ prompts and generating
$G=8$ responses per prompt in each fixed-budget step. We evaluate overall and Level-5
accuracy on the disjoint MATH500 subset \cite{lightman2024verify} every $10$ steps, and
final-checkpoint accuracy on GSM8K \cite{cobbe2021gsm8k}, Minerva Math
\cite{lewkowycz2022minerva}, and OlympiadBench \cite{he2024olympiadbench}. Further evaluations
use Qwen2.5-7B and a $1.5$B policy \cite{qwen2024qwen25}, a $1200$-prompt pool, and longer
horizons. Appendix~F specifies optimization, decoding, and hardware.

\paragraph{Priors and selection rules.}
Qwen2.5-3B-Instruct is the default anchor in Eq.~\eqref{eq:prior}. Alternative priors use a
2PL item-response-theory bank fitted to a model--prompt response matrix, or reasoning
lengths from Qwen3-0.6B \cite{yang2025qwen3} mapped to success rates by a logistic fit.
We compare uniform sampling, MoPPS \cite{qu2025mopps}, GRESO \cite{zheng2025greso}, and DAPO
\cite{yu2025dapo} in a shared trainer. Model, optimizer, LoRA settings, horizon, and evaluation
are held fixed. The DAPO comparisons share a four-round refill cap and retain $B$ non-silent
groups per update.

The component comparisons use \textbf{online-NP}, which initializes every prompt at
$\mathrm{Beta}(1,1)$ and otherwise follows the same score and online update, and
\textbf{prior-only}, which samples from the fixed top half of the initial ranking.
\textbf{Length-prior}, \textbf{2PL bank}, and \textbf{bank-sharp} replace the initialization
and retain online feedback. Online-NP changes the initial pseudo-count mass from about
$4$ to $2$ as well as removing prompt-specific information; prior-only also changes sampling
within the selected band.

\paragraph{Metrics and run sets.}
\textbf{Silent@10} is the fraction of generated candidate groups with uniform rewards during
the first ten update steps; silent@200 uses the same cumulative definition over the longer
horizon. \textbf{Waste@30} counts responses in silent candidate groups through step~$30$.
The denominator includes every generated candidate group, including groups later filtered
by dynamic sampling. We report generation separately from the responses retained for an
update, so refill overhead is included in the DAPO comparison.

The main comparison averages three runs per configuration. The expanded \method{}--online-NP
comparison uses sixteen runs per method, with the same offline prior reused across runs.
Table~\ref{tab:sig} reports uncertainty and statistical tests; Appendix~F records the additional
run sets and timing measurements.

\begin{table*}[t]
\centering
\fontsize{9}{10.4}\selectfont
\setlength{\tabcolsep}{4.6pt}
\renewcommand{\arraystretch}{1.0}
\begin{tabular}{@{}lcccccccc@{}}
\toprule
\multirow{2}{*}{Method} & \multirow{2}{*}{Generated R/step} &
\multicolumn{2}{c}{Rollout efficiency \da} &
\multicolumn{5}{c}{Accuracy \ua} \\
\cmidrule(lr){3-4}\cmidrule(lr){5-9}
 & & silent@10 & waste@30 & MATH500 & Level-5 & GSM8K & Minerva & Olympiad \\
\midrule
\rowcolor{cGroup}
\multicolumn{9}{l}{\textsc{No selection}}\\
random (uniform) & $64$ & $0.379$\dn{17.5} & $749$\dn{440} & $0.545$ & $0.256$ & $0.785$ & $0.134$ & $0.149$ \\
\addlinespace[1.5pt]
\rowcolor{cGroup}
\multicolumn{9}{l}{\textsc{Online selection methods}}\\
MoPPS \cite{qu2025mopps} & $64$ & $0.425$\dn{22.1} & $621$\dn{312} & $0.561$ & $0.289$ & $0.817$ & $0.178$ & $0.179$ \\
GRESO \cite{zheng2025greso} & $64$ & $0.413$\dn{20.8} & $701$\dn{392} & $0.525$ & $0.236$ & $0.787$ & $0.148$ & $0.142$ \\
DAPO \cite{yu2025dapo} & ${\sim}138$ & $0.362$\dn{15.8} & $1565$\dn{1256} & $0.605$ & $0.333$ & $0.847$ & $0.168$ & $0.244$ \\
\addlinespace[1.5pt]
\rowcolor{cGroup}
\multicolumn{9}{l}{\textsc{Initialization and update variants}}\\
online-NP (posterior only) & $64$ & $0.204$ & $309$ & $0.528$ & $0.249$ & $0.815$ & $0.126$ & $0.163$ \\
prior-only (frozen) & $64$ & $0.246$\dn{4.2} & $413$\dn{104} & $0.533$ & $0.241$ & $0.788$ & $0.145$ & $0.154$ \\
length prior & $64$ & $\numLsil$\dn{1.3} & $\numLwst$\dn{243} & $\numLacc$ & $0.294$ & $0.836$ & $0.164$ & $0.204$ \\
\addlinespace[1.5pt]
\rowcolor{cGroup}
\multicolumn{9}{l}{\textsc{Verified offline priors}}\\
\rowcolor{cOurs}
\textbf{\method{} (ours)} & $64$ & $\best{0.138}$\up{6.7} & $\best{280}$\up{29} & $0.561$ & $0.303$ & $0.817$ & $0.150$ & $0.199$ \\
\rowcolor{cOursLite}
2PL bank (offline IRT) & $64$ & $\best{0.125}$\up{7.9} & $\best{235}$\up{74} & $0.593$ & $0.306$ & $0.854$ & $0.130$ & $0.227$ \\
\rowcolor{cOursLite}
2PL bank + sharpening & $64$ & $\best{0.167}$\up{3.8} & $\best{312}$\dn{3} & $0.582$ & $0.313$ & $0.831$ & $0.184$ & $0.209$ \\
\addlinespace[1.5pt]
\rowcolor{cGroup}
\multicolumn{9}{l}{\textsc{Combination with dynamic sampling}}\\
\rowcolor{cOurs}
\textbf{\method{} + DAPO} & ${\sim}123$ & $0.158$\up{4.6} & $541$\dn{232} & $0.605$ & $0.336$ & $0.829$ & $0.157$ & $0.200$ \\
\bottomrule
\end{tabular}
\caption{Prompt selection on Qwen2.5-Math-7B. Entries are means over three runs;
accuracy values are fractions. Generated R/step counts all candidate responses, including
those filtered by dynamic sampling. Subscripts show differences from online-NP, in percentage
points for silent@10 and response counts for waste@30; green indicates a reduction and red
an increase. Among fixed-budget methods, bold highlights the best efficiency values and
methods with overlapping observed run ranges.}
\label{tab:main}
\end{table*}

\section{Results}

\subsection{Early rollout utilization}

\begin{figure}[b]
\centering
\includegraphics[width=\columnwidth]{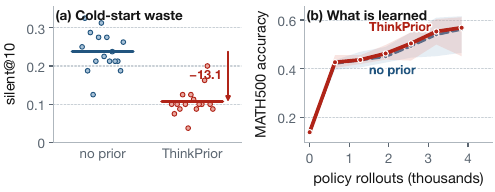}
\caption{Training behavior of \method{} and online-NP. (a) Early silent-group fractions
for individual runs, with horizontal marks indicating the means. (b) MATH500 learning
curves as a function of generated policy rollouts. Both methods use sixteen runs.}
\label{fig:l5}
\end{figure}

The verified priors yield the lowest early silent-group fractions among the fixed-budget
configurations in Table~\ref{tab:main}. \method{} improves on uniform sampling and the
history-based selectors, while the IRT priors achieve similar or lower waste. The common
feature is an initial estimate that distinguishes prompts before target-policy observations
accumulate.

The expanded comparison with online-NP keeps the score and update rule fixed while
replacing the uninformative initialization. It shows a
$55\%$ reduction in early silent groups (Table~\ref{tab:sig}). Figure~\ref{fig:l5}(a) displays
the separation across individual runs; panel (b) shows closely tracking MATH500 learning
curves at the same generated-response budget. Final performance on the additional reasoning
benchmarks remains competitive (Table~\ref{tab:main}). Together, the utilization and learning
curves show that useful early groups can be obtained more reliably while preserving task
performance.

\begin{table}[b]
\centering
\fontsize{8}{9.2}\selectfont
\setlength{\tabcolsep}{1.6pt}
\renewcommand{\arraystretch}{1.02}
\begin{tabular}{@{}lcccc@{}}
\toprule
Method & silent@10 \da & waste@30 \da & MATH500 \ua & Level-5 \ua \\
\midrule
\rowcolor{cOurs}
\textbf{\method{}} & $\best{0.106}{\pm}0.036$ & $\best{266}{\pm}30$ &
$0.569{\pm}0.037$ & $0.306{\pm}0.037$ \\
online-NP & $0.238{\pm}0.051$ & $329{\pm}77$ & $0.562{\pm}0.042$ & $0.293{\pm}0.046$ \\
\midrule
$\Delta$ & $-13.1$\,pts & $-63$ & $+0.68$\,pts & $+1.26$\,pts \\
$95\%$ CI & $[-16.4,-9.9]$ & $[-106,-20]$ & $[-2.2,3.5]$ & $[-1.8,4.3]$ \\
Cohen's $d$ & $-2.95$ & $-1.07$ & $0.17$ & $0.30$ \\
Welch $p$ & ${<}10^{-4}$ & $0.007$ & $0.63$ & $0.40$ \\
Perm. $p$ & ${<}10^{-4}$ & $0.002$ & $0.32$ & $0.21$ \\
\addlinespace[2pt]
Ckpts led & -- & -- & $6/6$ & $3/6$ \\
\bottomrule
\end{tabular}
\caption{Detailed comparison of \method{} and online-NP over sixteen runs per method.
Values are means $\pm$ sample standard deviations; $\Delta$ is \method{} minus online-NP.
Confidence intervals and Welch tests are two-sided; permutation tests are one-sided with
$4\times10^5$ resamples. Differences and tests use unrounded values. Bold highlights the
efficiency improvements whose intervals exclude zero. Ckpts led counts the checkpoints at
which \method{} has higher mean accuracy, out of six evaluations.}
\label{tab:sig}
\end{table}

\subsection{Reducing generation with dynamic sampling}

DAPO converts the improvement in candidate selection into a reduction in generated responses.
With its useful update batch held fixed, \method{} uses $10.6\%$ fewer policy rollouts and
matches the baseline's mean MATH500 accuracy (Table~\ref{tab:money}). Most of the reduction
comes from avoiding silent candidates during refill. The increase in non-silent candidates
also leaves more responses available than the trainer needs for its update.

This distinction matters for interpreting the cost columns. A final refill can supply excess
non-silent groups, so the generated count includes responses that are neither silent nor used
for optimization. We retain all of them in the accounting. The measured generation counts
therefore reflect the actual sampling work performed by the shared refill procedure, with
the same number of responses used for policy updates in both arms. The corresponding
generation-time estimates and the separate anchor-preparation cost are detailed in Appendix~F.

\begin{table}[t]
\centering
\fontsize{8}{9.2}\selectfont
\setlength{\tabcolsep}{3.6pt}
\renewcommand{\arraystretch}{1.0}
\begin{tabular*}{\columnwidth}{@{\extracolsep{\fill}}lcc>{\columncolor{cOurs}}c@{}}
\toprule
Metric & DAPO & \multicolumn{2}{c}{$+$\,\method{}} \\
\cmidrule(l){3-4}
& & value & change \\
\midrule
MATH500 \ua & $0.605$ & $0.605$ & $0.0$\,pts \\
\addlinespace[3pt]
Generated rollouts \da & $8256$ & $7381$ & $-10.6\%$ \\
Silent generated rollouts \da & $3325$ & $2131$ & $-35.9\%$ \\
Non-silent generated rollouts \ua & $4931$ & $5251$ & $+6.5\%$ \\
Rollouts used for updates & $3840$ & $3840$ & $0$ \\
Estimated generation GPU-hours \da & $1.87$ & $1.68$ & $-0.20$ \\
\addlinespace[3pt]
silent@10 \da & $0.362$ & $0.158$ & $-20.4$\,pts \\
waste@30 \da & $1565$ & $541$ & $-65.4\%$ \\
\bottomrule
\end{tabular*}
\caption{Generation cost of DAPO with and without \method{}. Values are means over
three $60$-step runs. Both methods use $3840$ responses for updates; candidate counts include
all generated responses, including filtered groups and excess groups from the final refill.
Generation time is estimated at $0.817$ seconds per response. Means are rounded independently.}
\label{tab:money}
\end{table}

\paragraph{Amortizing the anchor pass.}
The main-pool probe uses $4000$ anchor responses, while the DAPO comparison saves about
$875$ policy responses per run. Their costs depend on different models, so the response
counts alone do not determine the preparation overhead. Let $C_A$ denote the one-time
anchor-generation cost, $c_\pi$ the average cost per policy response, and $J$ the number of
runs reusing the prior. Under the per-response accounting of Table~\ref{tab:money}, the
policy-generation savings cover the probe when
\begin{equation}
\frac{C_A}{J}<\Delta R\,c_\pi,
\label{eq:amortization}
\end{equation}
where $\Delta R$ is the mean reduction in generated policy responses. A single run bears
the full preparation cost; repeated runs on the same pool share it. Anchor choice
consequently matters through both its preparation cost and its ranking quality.

\subsection{Across policies and training horizons}

The early reduction also appears with the general-purpose Qwen2.5-7B backbone
(Table~\ref{tab:gen}(a)), extending the result beyond the math-specialized policy.
Table~\ref{tab:gen} separates these policy comparisons from changes in anchor size, longer
training, and a larger prompt pool. The larger-pool runs reproduce the initial advantage;
their later trajectories test how far this benefit persists as the policy and the available
training prompts evolve. We examine that behavior below.

\begin{table}[t]
\centering
\fontsize{8}{9.2}\selectfont
\setlength{\tabcolsep}{1.5pt}
\renewcommand{\arraystretch}{1.02}
\begin{tabular}{lcccc}
\toprule
\rowcolor{cGroup}
\multicolumn{5}{l}{\textbf{(a) Policy backbone and size}}\\
\multirow{2}{*}{Setting} & \multicolumn{2}{c}{Accuracy \ua} & \multicolumn{2}{c}{silent@10 \da}\\
\cmidrule(lr){2-3}\cmidrule(lr){4-5}
 & \method{} & NP & \method{} & NP\\
\midrule
Qwen2.5-Math-7B (main) & $0.577$ & $0.527$ & \best{0.113} & $0.271$\\
Qwen2.5-7B (general)   & $0.556$ & $0.575$ & \best{0.067} & $0.142$\\
Qwen2.5-1.5B policy    & $0.289$ & $0.263$ & -- & --\\
Qwen2.5-Math-7B (rerun) & $0.586$ & $0.582$ & \best{0.125} & $0.208$\\
\midrule
\rowcolor{cGroup}
\multicolumn{5}{l}{\textbf{(b) Anchor models}}\\
\addlinespace[1pt]
Anchor $T$ & Size & $\overline{\hat\varphi}$ & MATH500 & silent@10 \da\\
\midrule
Qwen2.5-Instruct   & $1.5$B & $0.172$ & $0.587$ & $0.106$\\
\rowcolor{cOurs}
Qwen2.5-Instruct   & $3$B   & $0.252$ & $0.577$ & $0.113$\\
Qwen2.5-Math-Inst. & $7$B   & $0.276$ & $0.528$ & $0.100$\\
\midrule
\rowcolor{cGroup}
\multicolumn{5}{l}{\textbf{(c) Exploratory accuracy-drop events at step $300$}}\\
\addlinespace[1pt]
Selection rule & \multicolumn{2}{c}{no KL} & \multicolumn{2}{c}{$\beta{=}0.04$}\\
\midrule
\rowcolor{cOurs}
\method{} & \multicolumn{2}{c}{\best{0/6}} & \multicolumn{2}{c}{\best{0/4}}\\
online-NP & \multicolumn{2}{c}{$4/6$} & \multicolumn{2}{c}{$2/4$}\\
\midrule
\rowcolor{cGroup}
\multicolumn{5}{l}{\textbf{(d) Training with a $1200$-prompt pool}}\\
\addlinespace[1pt]
Pool $=1200$ & \multicolumn{2}{c}{wasted \da} & silent@10 \da & silent@200 \da\\
\midrule
\rowcolor{cOurs}
\method{} & \multicolumn{2}{c}{\lpOurWaste} & \best{\lpOurSilentA} & \lpOurSilentB\\
online-NP & \multicolumn{2}{c}{$\lpNpWaste$} & $\lpNpSilentA$ & $\lpNpSilentB$\\
\bottomrule
\end{tabular}
\caption{Additional evaluations of \method{}; NP denotes online-NP. (a) Policy backbone and size.
(b) Anchor models. (c) Exploratory counts of peak-to-final accuracy drops exceeding ten
percentage points at step $300$, with and without reference KL. (d) A $1200$-prompt pool
trained for $200$ steps, with cumulative wasted-response counts. Panels (a,d) average three
runs and (b) averages two; (c) reports event counts. Training settings and event definitions
are given in Appendices D--F.}
\label{tab:gen}
\end{table}

\section{Analysis}

\subsection{How the ranking evolves}

The trajectories in Figure~\ref{fig:eff} separate the effects of initialization and accumulated
feedback. Uniform sampling continues to generate silent groups at a similar rate. Online-NP
improves as observations distinguish prompts, whereas \method{} starts with a useful ranking.
The frozen-prior comparison in Table~\ref{tab:main} also improves on uniform sampling,
showing that the external scores contain usable information before online adaptation.
Combining this initialization with feedback yields the strongest early efficiency among
these component configurations.

On the small pool, the gap subsequently narrows and \method{}'s cumulative full-run waste
slightly exceeds online-NP's. An initially informative prompt may become consistently
solvable, while its estimate still aggregates earlier policy outcomes. The score changes
with both its mean and concentration, as Eq.~\eqref{eq:mean-update} shows; increasing evidence
alone does not keep a prompt useful. The observed trajectories establish an early selection
benefit and motivate evaluating it over longer horizons. They also explain the role of DAPO
in the composition: filtering removes remaining silent groups before an update, even when
the prior-informed ranking admits them.

\begin{figure}[b]
\centering
\includegraphics[width=\columnwidth]{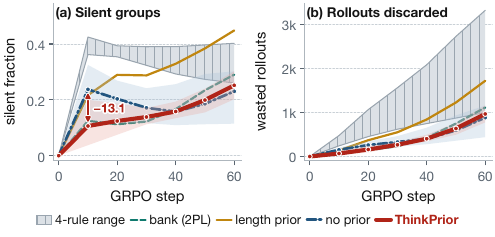}
\caption{Rollout utilization over training. Left: silent-group fractions in successive
training windows. Right: cumulative responses belonging to silent groups. \method{}
and online-NP curves average sixteen runs; the other methods use three runs. Shaded
regions show observed ranges.}
\label{fig:eff}
\end{figure}

\subsection{What information transfers from the anchor?}

Larger anchors do not markedly improve the early silent-group rate in Table~\ref{tab:gen}(b).
The smallest anchor tested already identifies useful candidates. To understand what this
signal must capture, we compare the priors' rankings with reference success rates over the
whole pool and within the learnable band, defined by $U(p)>0.8$. Reference rates use
$32$ samples per prompt and serve only to characterize the pool; the selection rules do
not access them.

\begin{table}[t]
\centering
\small
\begin{tabular}{lcc}
\toprule
prior & all $250$ & band ($n{=}76$) \\
\midrule
anchor (ours)   & $0.62$ & $0.38$ \\
2PL bank        & $0.76$ & $0.45$ \\
anchor length   & $0.27$ & $0.25$ \\
\bottomrule
\end{tabular}
\caption{Spearman correlation with the reference pass rate, over the whole pool and inside the
learnable band.}
\label{tab:band}
\end{table}

The IRT bank has a clear advantage in global rank correlation, but that advantage shrinks
inside the learnable band (Table~\ref{tab:band}). Ordering very easy and very hard prompts
accurately contributes to the global statistic without necessarily changing which candidates
are selected for an update. This is consistent with the anchor and bank priors achieving
similar early utilization despite different global rankings. The comparison identifies a
useful target for difficulty estimation: resolving the region where the selector operates.
An anchor also obtains its estimates directly on a new pool, whereas an IRT fit requires a
model--prompt response matrix.

The length prior provides a complementary test. It has lower construction cost and lower
global calibration error, yet selects less useful initial batches. Its main signal is whether
the anchor reaches the generation limit; among completed traces, length adds little
discrimination. This separates a coarse indication of difficulty from the information needed
to locate informative prompts. Appendix~B reports the calibration, truncation, and AUC
measurements. Verified success rates provide the more effective selection signal in this
comparison.

\paragraph{Coverage of online feedback.}
Only selected prompts can correct their estimates. An anchor score near either extreme can
place an otherwise useful prompt low in the ranking; until it is sampled, its score retains
that initialization. Selected prompts, meanwhile, accumulate the concentration rewarded by
Proposition~\ref{prop:disp}. The finite prior strength controls how readily new observations
can revise an error, but does not itself cause an overlooked prompt to be sampled.

This affects a substantial part of the reference learnable band: $21$ of its $76$ prompts
receive an anchor score of either zero or one, putting their initial $U_\beta$ near its floor.
The errors occur at both ends of the success-rate scale, so a prompt can be overlooked
because the anchor finds it either too easy or too hard. Appendix~B gives the breakdown.

The recorded trace nevertheless visits $69$ band prompts by step~$30$. Combining these
counts implies that at least $14$ of the initially extreme-scored band prompts have been
selected. The trace therefore shows that initially low-ranked prompts can enter the training
batches as scores evolve. The remaining prompts may be delayed or persistently excluded;
these aggregate counts do not distinguish the two outcomes.

\subsection{Persistence over longer training}

The larger pool retains the early advantage and reverses the mean full-run waste ordering
seen on the small pool: \method{} now wastes fewer responses overall. Late outcomes overlap
across runs, so pool enlargement extends the observed pattern without identifying which
combination of prompt coverage and policy change drives it. This distinction matters because
top-$B$ selection permits revisits; the number of training steps alone does not establish
whether the pool has been exhausted.

Longer trajectories also show large evaluation-accuracy declines in some online-NP runs.
Table~\ref{tab:gen}(c) reports the exploratory relative-drop counts, with and without a
reference-KL term. These observations broaden the evaluation beyond cold start, while the
endpoint definitions and per-run results in Appendices~D and~E determine how the comparisons
should be interpreted. The initial reduction in silent groups is the pattern reproduced
across the tested policies and pool sizes.

\section{Related Work}

\subsection{Efficient RLVR}

RLVR learns from automatically verified outcomes \cite{lambert2024tulu3,guo2025deepseekr1},
and selective rollout methods seek to make better use of the resulting training budget.
DAPO \cite{yu2025dapo} builds update batches from groups with reward variation. GRESO
\cite{zheng2025greso} predicts persistent silent outcomes from reward history while retaining
an exploration probability; MoPPS \cite{qu2025mopps} uses posterior sampling, and online
difficulty filtering \cite{bae2025odf} prioritizes intermediate success rates.

PCL \cite{gao2025pcl} learns a prompt curriculum, while GPS \cite{qu2026gps} shares
predictions across prompts and combines them with history-aware selection. sGPO
\cite{sudalairaj2026sgpo} obtains its initial profile by running the target policy on the
pool. \method{} supplies external observations before this target-policy history is
available, then incorporates online outcomes into its selection score.

\subsection{Difficulty Estimation}

Item response theory models item difficulty and discrimination \cite{lord1980irt} and has
been used for NLP difficulty estimation and compact evaluation
\cite{lalor2016irt,lalor2019irt,polo2024tiny}. It provides a way to estimate prompt properties
from a model--prompt response matrix. Reasoning traces \cite{wei2022cot} and test-time
scaling \cite{snell2024scaling} offer signals based on the computational effort associated
with a problem. Our bank and length-prior comparisons examine whether these difficulty
estimates identify the region likely to produce mixed-reward groups.

\subsection{Curricula and Data Valuation}

Curriculum learning orders examples by difficulty or learning progress
\cite{bengio2009curriculum,kumar2010spl,graves2017acl}. Prioritized replay and active learning
allocate computation according to estimated example value
\cite{schaul2016per,jiang2021plr,settles2009active}. For language models, importance resampling,
filtering, and influence estimates support pretraining and instruction-data selection
\cite{xie2023dsir,li2024superfiltering,xia2024less}. RHO-LOSS \cite{mindermann2022rholoss}
uses a smaller reference model to prioritize useful examples. \method{} applies external
information to a group-level objective: the probability that a prompt produces a relative
reward signal during policy training.

\FloatBarrier
\section{Conclusion}

\method{} reduces the cost of obtaining useful rollout groups at cold start. Verified
external-model responses initialize a posterior score grounded in GRPO's relative reward
signal, and training outcomes refine the estimates. The experiments show improved early
utilization and fewer generated responses when combined with DAPO, while retaining comparable
task performance. The prior comparisons further show why difficulty estimates should be
assessed in the region where selection operates, alongside global calibration or ranking.

Future work should examine the interaction between prompt coverage and policy drift, and
evaluate the approach across domains, model families, and full-parameter training.
The preprint and supporting figures are available at
\url{https://shatianming5.github.io/thinkprior/}.

\clearpage

\bibliography{references}

\clearpage
\appendix

\section{Appendix A: Proofs}

Throughout, $G\ge2$ is the group size, $v(\cdot)\in\{0,1\}$ the verifier, $r_i=v(y_i)$ the reward
of the $i$-th rollout in a group, $C=\sum_{i=1}^G r_i$, and $p$ the prompt's pass rate under the
sampling policy, so that $C\sim\mathrm{Binomial}(G,p)$. We write $\bar r=C/G$ and use the sample
standard deviation $\operatorname{std}(r)^2=\frac{1}{G-1}\sum_i(r_i-\bar r)^2$, matching the
implementation. All statements are for the exact normalizer, i.e.\ in the limit
$\varepsilon\to0$ of Eq.~\eqref{eq:adv}. Only Proposition~\ref{prop:silent}(iii) depends on
$\varepsilon$ at all: the finite normalizer scales it by $(1+\varepsilon\sqrt{G})^{-2}$ in the
worst case $C\in\{1,G-1\}$, so at $G=8$ and $\varepsilon=10^{-4}$ it reads $6.996$ rather than
$7$. Every other identity here is exact in $\varepsilon$.

\subsection{Proof of Proposition~\ref{prop:silent}}

\paragraph{(i) Silence and its probability.}
$A_i=0$ for all $i$ if and only if $r_i=\bar r$ for all $i$, which for binary $r_i$ holds if and
only if $C\in\{0,G\}$. The two events are disjoint, with $\Pr[C=0]=(1-p)^G$ and
$\Pr[C=G]=p^G$, so $\Pr[\text{silent}]=p^G+(1-p)^G=s(p)$ and $\Pr[\text{not silent}]=U(p)$.

\paragraph{(ii) Shape of $U$.}
Differentiating Eq.~\eqref{eq:learn} twice gives, for $p\in(0,1)$,
\[
U'(p)=G\left[(1-p)^{G-1}-p^{\,G-1}\right],
\]
\[
U''(p)=-G(G-1)\left[p^{\,G-2}+(1-p)^{G-2}\right].
\]
Since $G\ge2$, the bracket in $U''$ is strictly positive on $(0,1)$, so $U''<0$ and $U$ is
strictly concave. Exchanging $p$ and $1-p$ leaves Eq.~\eqref{eq:learn} unchanged, so
$U(p)=U(1-p)$ and $U$ is symmetric about $p=1/2$. For $p<1/2$ we have $1-p>p\ge0$ and
$G-1\ge1$, hence $(1-p)^{G-1}>p^{G-1}$ and $U'(p)>0$; by symmetry $U'(p)<0$ for $p>1/2$.
Therefore $U$ is strictly increasing on $[0,1/2]$, strictly decreasing on $[1/2,1]$, and so is a
strictly decreasing function of $|p-1/2|$, with a unique maximum at $p=1/2$ of value
$U(1/2)=1-2\cdot2^{-G}=1-2^{1-G}$. Finally $U(0)=U(1)=0$, while for $p\in(0,1)$ both $p<1$ and
$1-p<1$ give $p^G+(1-p)^G<p+(1-p)=1$, so $U>0$; the zeros are exactly $\{0,1\}$.

\paragraph{(iii) Advantage mass.}
Condition on the group not being silent, so $1\le C\le G-1$. With $\bar r=C/G$,
\[
\sum_{i=1}^G (r_i-\bar r)^2
= C\Bigl(1-\tfrac{C}{G}\Bigr)^{2}+(G-C)\Bigl(\tfrac{C}{G}\Bigr)^{2}
= \frac{C(G-C)}{G},
\]
because $C(G-C)^2+(G-C)C^2=C(G-C)\bigl[(G-C)+C\bigr]=C(G-C)G$. Hence
$\operatorname{std}(r)^2=\frac{C(G-C)}{G(G-1)}>0$ and
\[
\sum_{i=1}^G A_i^2=\frac{\sum_i (r_i-\bar r)^2}{\operatorname{std}(r)^2}=G-1,
\]
which does not depend on $C$. Every non-silent group therefore carries the same total squared
advantage, and combining with (i),
$\mathbb{E}\bigl[\sum_i A_i^2\bigr]=(G-1)\Pr[\text{not silent}]=(G-1)\,U(p)$. $\square$

\medskip
\noindent Maximizing $U(p)$ therefore maximizes the expected squared-advantage sum.
The policy gradient additionally weights each advantage by the sampled response's
score-function gradient, whose direction and magnitude depend on the response.

\subsection{A closed form for posterior expected learnability}

\begin{lemma}\label{lem:beta}
If $p\sim\mathrm{Beta}(\alpha,\beta)$ with $\alpha,\beta>0$, then for integer $G\ge1$
\[
\mathbb{E}\bigl[p^{\,G}\bigr]=\frac{(\alpha)_G}{(\alpha+\beta)_G},
\quad (a)_G:=\prod_{j=0}^{G-1}(a+j).
\]
\end{lemma}

\begin{proof}
Writing $B(\cdot,\cdot)$ for the Beta function,
\[
\mathbb{E}\bigl[p^{G}\bigr]
=\frac{1}{B(\alpha,\beta)}\int_0^1\! p^{\alpha+G-1}(1-p)^{\beta-1} dp
\]
equals $B(\alpha+G,\beta)/B(\alpha,\beta)$, and expanding both Beta functions in Gamma functions
gives
$\frac{\Gamma(\alpha+G)\Gamma(\alpha+\beta)}{\Gamma(\alpha)\Gamma(\alpha+\beta+G)}
=\frac{(\alpha)_G}{(\alpha+\beta)_G}$.
\end{proof}

\noindent Applying Lemma~\ref{lem:beta} to $U(p)=1-p^G-(1-p)^G$ and using that
$1-p\sim\mathrm{Beta}(\beta,\alpha)$ yields Eq.~\eqref{eq:ubeta} directly.

\subsection{Proof of Proposition~\ref{prop:disp}}

Fix $\mu=\alpha_x/(\alpha_x+\beta_x)\in(0,1)$ and $m=\alpha_x+\beta_x\in(0,\infty)$, so that
$\alpha_x=\mu m$ and $\beta_x=(1-\mu)m$. Define
\[
F(m)=\prod_{j=0}^{G-1}\frac{\mu m+j}{m+j},
\qquad
\tilde F(m)=\prod_{j=0}^{G-1}\frac{(1-\mu)m+j}{m+j},
\]
so that $U_\beta=1-F(m)-\tilde F(m)$ by Lemma~\ref{lem:beta}.

\paragraph{Strict inequality against the plug-in score.}
$U$ is strictly concave by Proposition~\ref{prop:silent}(ii) and the posterior is
non-degenerate whenever $m<\infty$, so Jensen's inequality is strict:
$U_\beta=\mathbb{E}[U(p)]<U(\mathbb{E}[p])=U(\mu)$. Scoring by $U_\beta$ therefore
penalizes dispersion relative to evaluating $U$ at the posterior mean.

\paragraph{Monotonicity in the concentration $m$.}
Differentiating the logarithm term by term,
\[
\frac{d}{dm}\log F(m)=\sum_{j=0}^{G-1}\left[\frac{\mu}{\mu m+j}-\frac{1}{m+j}\right].
\]
The $j=0$ term is $\frac{\mu}{\mu m}-\frac{1}{m}=0$. For $j\ge1$,
\[
\frac{\mu}{\mu m+j}-\frac{1}{m+j}
=\frac{j(\mu-1)}{(\mu m+j)(m+j)}<0
\]
since $\mu<1$. As $G\ge2$ there is at least one such term, so $\frac{d}{dm}\log F<0$ and $F$ is
strictly decreasing in $m$. Replacing $\mu$ by $1-\mu$, which is also in $(0,1)$, shows
$\tilde F$ is strictly decreasing in $m$ as well. Hence $U_\beta=1-F-\tilde F$ is strictly
increasing in $m$.

\paragraph{Limits.}
As $m\to\infty$ each factor $\frac{\mu m+j}{m+j}\to\mu$, so $F\to\mu^G$ and
$\tilde F\to(1-\mu)^G$, giving $U_\beta\to U(\mu)$: the posterior concentrates and the score
recovers the plug-in value. As $m\to0$ the $j=0$ factor equals $\mu$ exactly while every
$j\ge1$ factor tends to $1$, so $F\to\mu$ and $\tilde F\to 1-\mu$, whence
$U_\beta\to1-\mu-(1-\mu)=0$. $\square$

\medskip
\noindent At a fixed mean, a more concentrated posterior gives a higher expected
learnability. Our initialization sets the concentration to $\kappa+2\epsilon_0\approx4$
at an anchor-informed mean, while the uninformative ablation uses concentration $2$
and mean $1/2$. The experimental comparison thus varies both the initial mean and
concentration, as discussed with the ablation results.

\subsection{Proof of Proposition~\ref{prop:cold}}

Under $\alpha_x=\beta_x=1$ the posterior is uniform on $[0,1]$, and Lemma~\ref{lem:beta} gives
$\mathbb{E}[p^G]=\frac{(1)_G}{(2)_G}=\frac{G!}{(G+1)!}=\frac{1}{G+1}$, and likewise
$\mathbb{E}[(1-p)^G]=\frac{1}{G+1}$. Hence
\[
U_\beta(x)=1-\frac{2}{G+1}\qquad\text{for every }x,
\]
which at $G{=}8$ equals $7/9\approx0.778$. More generally, if the initialization assigns the same
pair $(\alpha,\beta)$ to every prompt, then Eq.~\eqref{eq:ubeta} evaluates to the same number for
every prompt, so $U_\beta$ is constant on the pool and the arg-max is determined entirely by the
tie-breaking order. $\square$

\section{Appendix B: Extended analysis}

\paragraph{Training-window statistics.}
Uniform sampling maintains a silent-group rate of $37.5\%$--$40.8\%$ across windows.
Online-NP reaches $10.7\%$ after twenty to thirty steps, and its gap from \method{} narrows
by about step~$40$ on the small pool. Full-run cumulative waste is $968$ responses for
\method{} and $885$ for online-NP; late training pass rates are $0.858$ and $0.841$,
respectively. These trajectories distinguish the initial selection benefit from full-run
utilization.

\paragraph{Calibration and the length signal.}
The full-pool calibration MAE is $0.12$ for the length prior and $0.18$ for the verified
anchor prior. The length-prior comparison uses $800$ generations, compared with $4000$
for the reusable verified probe. The logistic length fit obtains most of its signal from
truncation: the anchor reaches its $2560$-token limit on $155$ of the $250$ prompts. Those
prompts have mean reference pass rate $0.112$, compared with $0.205$ for the remaining $95$.
Among the completed traces, length has chance-level AUC. The truncation indicator alone
reaches AUC $0.651$, reciprocal length reaches $0.665$, and the $k=16$ verified probe reaches
$0.915$. The two priors also differ in anchor size and sample count.

\paragraph{Ranking within the learnable region.}
The pool covers MATH levels and subjects through spread sampling. Under the reference pass
rates, $76$ of the $250$ prompts satisfy $U(p)>0.8$. Table~\ref{tab:band} compares rank
correlations on this subset and on the whole pool. The bank's advantage over the anchor
shrinks from $0.14$ globally to $0.07$ inside the band. Thus, much of the global ranking
advantage concerns prompts outside the region most likely to supply mixed rewards.

\paragraph{Feedback coverage.}
The anchor assigns $\hat\varphi=0$ to $99$ prompts and $\hat\varphi=1$ to $19$. Of these,
$8$ and $13$, respectively, belong to the reference learnable band. In total, $21$ of the
$76$ band prompts therefore begin near the floor of $U_\beta$. The prior strength $\kappa=4$
limits the weight of these estimates once a prompt receives feedback, but top-$B$ selection
provides no explicit exploration or diversity term \cite{qu2026gps}. The recorded trace
visits $69$ band prompts by step~$30$. For the remaining prompts, the trace does not distinguish
delayed selection from continued exclusion.

\section{Appendix C: Limitations}

The experiments cover mathematical reasoning with one model family and LoRA. The larger-pool
and longer-horizon studies broaden this setting, while uniformly sampled pools, other domains,
and full-parameter training remain directions for evaluation. The small-pool experiment uses
$60$ steps, the longer trajectories extend to $300$, and the larger pool uses $1200$ prompts
for $200$ steps.

The expanded comparison supports the waste reductions; its accuracy differences remain
unresolved. The offline prior is held fixed across runs, so reported intervals capture
training variability conditional on that prior. The main multi-method comparison has three
runs per configuration.

External measurements are not refreshed during training. Online counts combine outcomes from
successive policy checkpoints and update only selected prompts, leaving initially
underestimated prompts dependent on whether they are revisited. The experiments do not
isolate the separate contributions of coverage and policy drift to the closing gap.

A shared exact-match verifier scores anchor responses, training rewards, and evaluation.
Its errors can consequently affect these stages in the same direction.

\section{Appendix D: Behaviour at $300$ steps}

We extend \method{} and online-NP to $300$ steps, each with and without a reference-KL term
at $\beta=0.04$.

\paragraph{Endpoint definitions.}
The original endpoint, fixed before running, was final MATH500 accuracy below $0.05$.
The KL-free comparison records $2/6$ events for online-NP and $0/6$ for \method{}
(Fisher two-sided $p=0.455$). After inspecting the trajectories, we additionally counted
peak-to-final drops exceeding $10$ percentage points. This includes online-NP runs ending
at $0.098$ and $0.290$ after drops of $52$ and $35$ points, giving $4/6$ versus $0/6$
($p=0.061$). Table~\ref{tab:gen}(c) reports this exploratory relative-drop endpoint.

\paragraph{Reference KL.}
Under the exploratory endpoint, online-NP records $4/6$ events without KL and $2/4$ with KL
(Fisher two-sided $p=1.00$); \method{} records zero in both settings. The zero-advantage
identities in Proposition~\ref{prop:silent} remain valid with reference KL, while the total
loss can receive an additional gradient from that term.

\paragraph{Interpretation.}
The original endpoint provides the planned comparison. Pooling the KL conditions gives
$6/10$ versus $0/10$ and $p=0.011$ under the exploratory endpoint, but combines runs designed
for different comparisons. Both this pooling and the relative-drop analysis are post hoc.
The twenty trajectories motivate further study of longer-horizon behavior rather than an
estimate of a component-specific stability effect.

\paragraph{Execution platforms.}
Seed-to-host assignments are identical across arms: seeds $0$ and $3$ run on one machine,
and the others on the second. KL-free online-NP events occur on both machines, at seed
indices $1/2$ and $3/4$.

\section{Appendix E: A larger prompt pool}
\label{app:bigpool}

\paragraph{Design.}
We rebuild the MATH pool using the same spread-sampling rule, increasing it from $250$ to
$1200$ prompts ($4.8\times$), and run \method{} and online-NP for $200$ steps.
The anchor pass is repeated on the full pool with Qwen2.5-3B-Instruct, $k=16$, and the same
verifier and decoding. Since top-$B$ selection permits revisits, $1200/8=150$ and
$250/8\approx31$ are coverage bounds only under sampling without replacement; the horizon
does not establish pool exhaustion.

Only the anchor pass rates are required by these two arms. Reference pass rates and
per-prompt reasoning lengths were not remeasured for the additional $950$ prompts, so the
oracle, length-prior, and bank configurations were not run. Their unused fields retain
placeholders. Generation length is a global trainer constant, and the reference pass rate
is read only by the oracle rule. Learnable-band and headroom statistics are therefore
reported only for the original pool.

\paragraph{Prior coverage and consistency.}
The prior file covers $1200/1200$ prompt identifiers with no extras; all runs record
$\texttt{miss}=0$. The trainer's default $\hat\varphi=0.5$ for missing entries is therefore
unused. All runs share a host and start at MATH500 accuracy $0.126$. On the original $250$
prompts, the repeated anchor pass has mean score difference $-0.014$, with $119/250$ scores
identical, consistent with sampling variation at $k=16$.

\paragraph{Rollout utilization.}
At step~$10$, cumulative silent-group fractions are $\lpNpSilentA$ for online-NP and
$\lpOurSilentA$ for \method{}, closely matching $0.238$ and $0.106$ on the small pool.
The run ranges do not overlap at steps~$10$ ($\lpNpSilentASeeds$ versus
$\lpOurSilentASeeds$) or~$30$ ($\lpNpSilentCSeeds$ versus $\lpOurSilentCSeeds$).
At step~$200$, the mean fractions are $\lpNpSilentB$ and $\lpOurSilentB$.
Over $12800$ generated responses, \method{} wastes $\lpOurWaste$ compared with online-NP's
$\lpNpWaste$, a mean reduction of $\lpNetPct$. The corresponding late run ranges overlap:
waste is $\lpOurWasteSeeds$ versus $\lpNpWasteSeeds$, and silence is
$\lpOurSilentBSeeds$ versus $\lpNpSilentBSeeds$. Table~\ref{tab:gen}(d) reports the means.

\paragraph{Accuracy and sensitivity.}
The primary mean includes all three runs. One online-NP run reaches MATH500 accuracy $0.024$
and GSM8K, Minerva, and OlympiadBench accuracies of $0.006$, $0.000$, and $0.007$, despite
a training pass rate of $0.748$ in the same window. Its results remain in both the accuracy
and utilization summaries. The online-NP MATH500 values $\lpNpAccSeeds$ give a mean of
$\lpNpAcc$, compared with $\lpOurAcc$ for \method{}. Omitting the low run yields
$\lpNpAccSensitivity$, reported only as a post-hoc sensitivity calculation. The accuracy
comparison remains descriptive with this run-to-run variation.

\paragraph{Reporting rule.}
The comparison uses three seeds per arm. Before the runs finished, we specified disjoint
observed ranges at each reported endpoint as the criterion for support, with no seeds added
after inspecting outcomes and no significance test at $n=3$. Online-NP's unfinished seeds
needed to produce fewer non-silent responses than \method{}'s lowest count, $6976$.
Seed~$1$ finished at $6688$ and seed~$2$ at $7104$. Thus the early endpoints meet the
criterion, while step~$200$ remains a directional comparison of overlapping runs.

\section{Appendix F: Training and evaluation details}

\paragraph{Optimization and decoding.}
LoRA uses $r=32$, $\alpha=64$, and learning rate $3\!\times\!10^{-5}$. Rollouts use
temperature $0.9$, top-$p$ $1.0$, $512$ new tokens, and prompts truncated to $384$;
evaluation is greedy. The simplified GRPO objective uses one inner iteration, no reference
KL, no policy-ratio clipping, and gradients norm-clipped at $1.0$. It is evaluated on the
untempered log-likelihood without an importance ratio. All arms share this objective and
run on RTX-4090 class GPUs, with one job per card. Reference pass rates estimated from $32$
samples per prompt characterize the original pool and are not used for selection.

\paragraph{Baseline settings.}
The shared-trainer MoPPS and GRESO configurations start with $\alpha=\beta=1$, as does
online-NP. DAPO uses a four-round refill cap. Each update retains exactly $B$ non-silent
groups; surplus groups from the final refill remain included in the generated count.

\paragraph{Length prior.}
Qwen3-0.6B \cite{yang2025qwen3} makes one greedy pass with a $2560$-token cap. We measure
$\ell(x)$ between the think delimiters and fit a two-parameter logistic map to
$\hat\varphi$ on a random $10\%$ of the pool. The answer itself is not used to construct
this prior.

\paragraph{Run sets.}
Table~\ref{tab:main} averages three runs for each of eleven configurations ($33$ runs on one
machine). Table~\ref{tab:sig} expands \method{} and online-NP to sixteen runs each;
MATH500 ranges are $0.452$--$0.612$ and $0.466$--$0.616$, respectively. Prior-only ranges
from $0.476$ to $0.580$ in the main comparison. Policy and anchor comparisons in
Table~\ref{tab:gen}(a,b) come from an earlier run set, averaging three and two runs,
respectively. Its main-model row is reported separately from the primary comparison.

\paragraph{Generation timing and DAPO counts.}
Two sets of uncontended timing runs, of sizes three and two, give $0.822\pm0.039$ and
$0.808\pm0.020$ seconds per response. The pooled estimate, $0.817$ seconds, determines
Table~\ref{tab:money}'s generation-time estimates and excludes verification, policy updates,
and orchestration. DAPO's three runs have MATH500 accuracies $0.602/0.622/0.590$ and
$8192/8128/8448$ generated responses. With \method{}, these values are
$0.620/0.608/0.586$ and $7040/7872/7232$. Retaining eight groups per update gives $3840$
update-used responses in each sixty-step run.

\end{document}